\documentclass{article}

\usepackage{times}
\usepackage{graphicx} % more modern
\usepackage{subfigure}
\usepackage{amsmath}
\usepackage{amsfonts}
\usepackage{amssymb,amsthm}
\usepackage{xspace}
\usepackage{multirow}
\usepackage{natbib}
\usepackage{rotating}
\usepackage{xparse}
\usepackage{colortbl}
\usepackage{enumitem}
\usepackage{soul}

\newcommand{\eq}[2]{#2}
\newcommand{\textbsf}[1]{{\bf \textsf{\small #1}}}
\definecolor{Gray}{gray}{0.9}
\NewDocumentCommand{\rot}{O{45} O{1em} m}{\makebox[#2][l]{\rotatebox{#1}{#3}}}%
\newcolumntype{g}{>{\columncolor{Gray}}c}
\newcommand{\R}{\ensuremath{\mathbb{R}}}

\newcommand\TODO[2]{}
\renewcommand\TODO[2]{\noindent\textbf{\textcolor{red}{\large TODO[#1]: #2}}}

\usepackage{algorithm}
\usepackage{algorithmic}
\usepackage{tikz}
\usetikzlibrary{trees}
\tikzstyle{bag} = [text width=4.5em, text centered]
\tikzstyle{end} = [circle, fill, inner sep=0pt, minimum size=3pt]

\usepackage{hyperref}

\usepackage[accepted]{icml2015} 
\eq{}{\def\defn#1{\expandafter\def\csname\string#1\endcsname}\defn{@copyrightspace}{}}
\icmltitlerunning{Learning to Search Better than Your Teacher}

\newtheorem{definition}{Definition}

\newcommand{\ignore}[1]{}

\newcommand{\lts}{\textsc{l2s}\xspace}
\newcommand{\myalg}{\textsc{LOLS}\xspace}

\newcommand{\searn}{\textsc{Searn}\xspace}

\newcommand{\aggrevate}{\textsc{AggreVaTe}\xspace}

\def\bx{{\boldsymbol{x}} }
\def\by{{\boldsymbol{y}}}
\def\bw{{\boldsymbol{w}}}
\def\bc{{\boldsymbol{c}}}
\def\piin{{\pi_i^{\text{in}}}}
\newcommand{\piout}[1][i]{{\pi_{#1}^{\text{out}}}}
\def\piref{{\pi^{\text{ref}}}}

\def\epsclass{{\delta_{\mbox{\small class}}}}

\def\Qref{{Q^{\piref}}}

\newcommand{\E}{\ensuremath{\mathbb{E}}}
\newcommand{\state}[2]{\ensuremath{d^{#1}_{#2}}}
\newcommand{\pil}[1]{\hat{\pi}_{#1}}
\newcommand{\best}[1]{{\bf #1}}

\newtheorem{theorem}{Theorem}
\newtheorem{lemma}{Lemma}
\newtheorem{corollary}{Corollary}
\newcommand{\pibar}{\ensuremath{\bar{\pi}}}
\newcommand{\trainset}{\ensuremath{\Gamma}}
\newcommand{\explset}{\ensuremath{\mathcal{I}}}
\newcommand{\explore}{\ensuremath{\texttt{explore}}}
\newcommand{\numexpl}[1]{\ensuremath{n_{#1}}}
\newcommand{\epsfull}[1][N]{\ensuremath{\delta_{#1}}}
\renewcommand{\P}{\ensuremath{\mathbb{P}}}
\newcommand{\regret}{\ensuremath\mathrm{Regret}}
\newcommand{\CS}{\ensuremath{\textsc{cs}}}
\renewcommand{\H}{\ensuremath{\mathcal{H}}}

\begin{document}
\eq{
\twocolumn[
	\icmltitle{Learning to Search Better than Your Teacher}

% It is OKAY to include author information, even for blind
% submissions: the style file will automatically remove it for you
% unless you've provided the [accepted] option to the icml2015
% package.
	\icmlauthor{Kai-Wei Chang}{kchang10@illinois.edu}
	\icmladdress{University of Illinois at Urbana Champaign, IL}
	\icmlauthor{Akshay Krishnamurthy}{akshaykr@cs.cmu.edu}
	\icmladdress{Carnegie Mellon University, Pittsburgh, PA}
	\icmlauthor{Alekh Agarwal}{alekha@microsoft.com}
	\icmladdress{Microsoft Research, New York, NY}
	\icmlauthor{Hal Daum\'e III}{hal@umiacs.umd.edu}
	\icmladdress{University of Maryland, College Park, MD}
	\icmlauthor{John Langford}{jcl@microsoft.com}
	\icmladdress{Microsoft Research, New York, NY}

% You may provide any keywords that you 
% find helpful for describing your paper; these are used to populate 
% the "keywords" metadata in the PDF but will not be shown in the document
	\icmlkeywords{}

	\vskip 0.3in
]
}{
\onecolumn
\icmltitle{Learning to Search Better than Your Teacher}

% It is OKAY to include author information, even for blind
% submissions: the style file will automatically remove it for you
% unless you've provided the [accepted] option to the icml2015
% package.
	\icmlauthor{Kai-Wei Chang}{kchang10@illinois.edu}
	\icmladdress{University of Illinois at Urbana Champaign, IL}
	\icmlauthor{Akshay Krishnamurthy}{akshaykr@cs.cmu.edu}
	\icmladdress{Carnegie Mellon University, Pittsburgh, PA}
	\icmlauthor{Alekh Agarwal}{alekha@microsoft.com}
	\icmladdress{Microsoft Research, New York, NY}
	\icmlauthor{Hal Daum\'e III}{hal@umiacs.umd.edu}
	\icmladdress{University of Maryland, College Park, MD, USA}
	\icmlauthor{John Langford}{jcl@microsoft.com}
	\icmladdress{Microsoft Research, New York, NY}

% You may provide any keywords that you 
% find helpful for describing your paper; these are used to populate 
% the "keywords" metadata in the PDF but will not be shown in the document
	\icmlkeywords{}

	\vskip 0.3in

}

\begin{abstract}
  Methods for learning to search for structured prediction typically
  imitate a reference policy, with existing theoretical guarantees
  demonstrating low regret compared to that reference. This is
  unsatisfactory in many applications where the reference policy is
  suboptimal and the goal of learning is to improve upon it. Can
  learning to search work even when the reference is poor?

  We provide a new learning to search algorithm, \myalg, which
  does well relative to the reference policy, but \emph{additionally}
  guarantees low regret compared to \emph{deviations} from the learned
  policy: a local-optimality guarantee.  Consequently, \myalg can
  improve upon the reference policy, unlike previous algorithms. This
  enables us to develop \emph{structured contextual bandits}, a
  partial information structured prediction setting with many
  potential applications.
\end{abstract} 

\frenchspacing  % for consistency
\section{Introduction}
\label{sec:intro}
% TODO:

% - Online search based model

In structured prediction problems, a learner makes joint predictions
over a set of interdependent output variables and observes a joint
loss.  For example, in a parsing task, the output is a parse tree over
a sentence.  Achieving optimal performance commonly requires the
prediction of each output variable to depend on neighboring variables.
One approach to structured prediction is \emph{learning to search}
(\lts)
\cite{collins04incremental,daume05laso,daume09searn,ross11dagger,doppa14hcsearch,ross14aggrevate},
which solves the problem by:

\begin{enumerate}[itemsep=0.1em,topsep=0em,parsep=0.1em]
\item converting structured prediction into a search problem with
  specified search space and actions;
\item defining structured features over each state to
capture the interdependency between output variables; 
\item constructing a reference policy based on training data;
\item learning a policy that \emph{imitates} the reference policy.
\end{enumerate}

Empirically, \lts approaches have been shown to be competitive with
other structured prediction approaches both in accuracy and running
time %% with other structured prediction approaches (see
(see e.g.~\citet{daume14imperativesearn}). Theoretically, existing \lts
algorithms guarantee that if the learning step performs well, then the
learned policy is almost as good as the reference policy, implicitly
assuming that the reference policy attains good performance. Good
reference policies are typically derived using labels in the training
data, such as assigning each word to its correct POS tag. However,
when the reference policy is suboptimal, which can arise for reasons
such as computational constraints,
%% is often the case when due to
%% computational constraints in computing it, 
nothing can be said for
existing approaches.

This problem is most obviously manifest in a ``structured contextual
bandit''\footnote{The key difference from (1) contextual bandits is
  that the action space is exponentially large (in the length of
  trajectories in the search space); and from (2) reinforcement
  learning is that a baseline reference policy exists before learning
  starts.} setting. For example, one might want to predict how the
landing page of a high profile website should be displayed; this
involves many interdependent predictions: items to show, position and
size of those items, font, color, layout, etc.  It may be plausible to
derive a quality signal for the displayed page based on user feedback,
and we may have access to a reasonable reference policy (namely the
existing rule-based system that renders the current web page). But,
applying \lts techniques results in nonsense---learning something
almost as good as the existing policy is useless as we can just keep
using the current system and obtain that guarantee. Unlike the full
feedback settings, label information is not even available during
learning to define a substantially better reference. The goal of
learning here is to improve upon the current system, which is most
likely far from optimal. This naturally leads to the question:
\emph{is learning to search useless when the reference policy is
  poor?}

This is the core question of the paper, which we address first with a
new \lts algorithm, \myalg (Locally Optimal Learning to Search) in
Section~\ref{sec:search}.  \myalg operates in an online fashion and
achieves a bound on a convex combination of regret-to-reference and
regret-to-own-one-step-deviations.  The first part ensures that good
reference policies can be leveraged effectively; the second part
ensures that even if the reference policy is very sub-optimal, the
learned policy is approximately ``locally optimal'' in a sense made
formal in Section~\ref{sec:analysis}.

\myalg operates according to a general schematic that encompases many
past \lts algorithms (see Section~\ref{sec:search}), including
Searn~\cite{daume09searn}, DAgger \cite{ross11dagger} and AggreVaTe
\cite{ross14aggrevate}. A secondary contribution of this paper is a
theoretical analysis of both good and bad ways of instantiating this
schematic under a variety of conditions, including: whether the
reference policy is optimal or not, and whether the reference policy
is in the hypothesis class or not. We find that, while past algorithms
achieve good regret guarantees \emph{when the reference policy is
  optimal}, they can fail rather dramatically when it is not.
%
%As a review, most \lts algorithms operate according to the following
%schematic: use a \emph{roll-in} policy to sample a set of states. For
%each of these states, consider all possible actions. For each
%state/action pair, use a \emph{roll-out} policy to complete the
%search, eventually leading to an observed \emph{loss}. Use this loss
%to train a classifier that tries to choose low-cost actions from each
%state.  In \myalg, the roll-out is according to a mixture of the
%learned policy and the reference policy, \emph{with the mixture
%  imposed on a per-example level rather than a per-decision
%  level.}
%
%We analyze the space of possible roll-in/roll-out combinations (in
%Section~\ref{sec:analysis}).  Past \lts algorithms achieve a regret
%bound that guarantees the learned policy is not much worse than the
%reference policy.  However, when the reference policy is suboptimal,
%we might also wish for the learned policy to be \emph{locally
%  optimal:} namely, any modification to a \emph{single} prediction for
%each trajectory of the policy (a ``one-step deviation'') yields a
%worse policy.
%The main result is that \myalg satisfies a new regret
%guarantee: the regret to a combination of one-step deviations and the
%reference policy is bounded. 
\myalg, on the other hand, has superior performance to other \lts algorithms
when the reference policy
performs poorly but local hill-climbing in policy space is effective.
%,
%the performance of \myalg may be far better than other \lts
%algorithms. 
In Section~\ref{sec:exp}, we empirically confirm that
\myalg can significantly outperform the reference policy in
practice on real-world datasets.

In Section~\ref{sec:structbandit} we extend \myalg to address the
structured contextual bandit setting, giving a natural modification to
the algorithm as well as the corresponding regret analysis.
%% Modifying the algorithm is
%% straight-forward: instead of rolling out for each deviation, we simply
%% roll-out for one randomly chosen deviation and use existing unbiased
%% estimation techniques~\cite{doublerobust} to estimate the cost of each
%% possible deviation. We bound the effect of this on learning quality by
%% establishing a regret bound on the resulting procedure.

\tikzstyle{level 1}=[level distance=1.5cm, sibling distance=1cm]
\tikzstyle{level 2}=[level distance=1.5cm, sibling distance=0.75cm]
	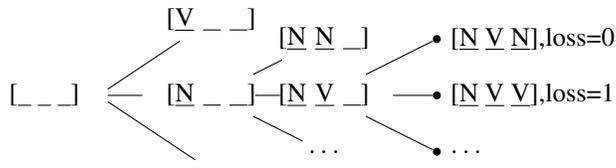
\begin{figure}[t]
		\begin{center}
		\begin{tikzpicture}[grow=right, sloped]
			\node[label=left:{[\underline{\ \ }\ \underline{\ \ }\ \underline{\ \ }]}]{}
			child {
				node[bag]{$\ldots$}
			}
			child {
			node[bag]{[\underline{N}\ \underline{\ \ }\ \underline{\ \ }]}
				child {
					node[bag]{$\ldots$}
				}
				child {
					node[bag]{[\underline{N}\ \underline{V}\ \underline{\ \ }]}
					child{
						node[end,label=right:{$\ldots$}]{}
					}
					child{
					node[end,label=right:{[\underline{N}\ \underline{V}\ \underline{V}],loss=1}]{}
					}
					child{
					node[end,label=right:{[\underline{N}\ \underline{V}\ \underline{N}],loss=0}]{}
					}
				}
				child {
					node[bag]{[\underline{N}\ \underline{N}\ \underline{\ \ }]}
				}
			}
			child {
				node[bag]{[\underline{V}\ \underline{\ \ }\ \underline{\ \ }]}
			};
%			\draw[decorate,decoration={brace,raise=2pt}] (0) -- node[label={[label distance=4pt]below:aaa}]{}  (5);
%			\coordinate (aux1) at ( $ (r1) + (0,5) $ );
%			\node(t1) at (aux1) {One-step deviation};
%			\draw[->,line width=1pt,dashed] (r1) -- (t1);
		\end{tikzpicture}
		\end{center}
		\vspace{-2em}
			\caption{An illustration of the search space of a sequential 
			tagging example that assigns a part-of-speech tag sequence to the 
			sentence ``John saw Mary.''
				Each state represents a partial labeling.
			The start state $b= [\underline{\ \ }\ \underline{\ \ }\ \underline{\ \ }]$ and the set of end states $E = \{
			[\underline{N}\ \underline{V}\ \underline{N}],
			[\underline{N}\ \underline{V}\ \underline{V}],\ldots\}$. Each end state is associated with a loss. A policy chooses an action at each state in the search space to specify the next state.
		}
		\label{fig:pos}
%% 		\label{fig:pos}
	\end{figure}

%% The proofs of our main results, and the details of the cost-sensitive
%% classifier used in experiments are deferred to the appendix. 
The
algorithm \myalg, the new kind of regret guarantee it satisfies, the
modifications for the structured contextual bandit setting, and all
experiments are new here.

\section{Learning to Search}
\label{sec:search}

% define structured prediction
A structured prediction problem consists of an \emph{input space}
$\mathcal{X}$, an \emph{output space} $\mathcal{Y}$, a fixed but
unknown distribution $\mathcal{D}$ over $\mathcal{X} \times
\mathcal{Y}$, and a non-negative \emph{loss function}
$\ell(\by^*,\hat{\by}) \rightarrow \mathbb{R}^{\geq0}$ which measures
the distance between the true ($\by^*$) and predicted ($\hat{\by}$)
outputs.  The goal of structured learning is to use $N$ samples $(\bx_i, \by_i)_{i=1}^N$
to learn a mapping
$f~:~\mathcal{X} \rightarrow \mathcal{Y}$ that minimizes the expected
structured loss under $\mathcal{D}$.

% define search space
In the learning to search framework, an input $\bx \in \mathcal{X}$
induces a search space, consisting of an initial state $b$ (which we
will take to also encode $\bx$), a set of end states and a transition
function that takes state/action pairs $s,a$ and deterministically
transitions to a new state $s'$.  For each end state $e$, there is a
corresponding structured output $\by_e$ and for convenience we define
the loss $\ell(e) = \ell(\by^*,\by_e)$ where $\by^*$ will be clear
from context.  We futher define a feature generating function $\Phi$
that maps states to feature vectors in $\mathbb{R}^d$.  The features
express both the input $\bx$ and previous predictions (actions).
Fig.~\ref{fig:pos} shows an example search
space\footnote{\citet{doppa14hcsearch} discuss several approaches for
  defining a search space.  The theoretical properties of our approach
  do not depend on which search space definition is used.}.

An agent follows a \emph{policy} $\pi \in \Pi$, which chooses an
\emph{action} $a\in A(s)$ at each non-terminal state $s$.  An action
specifies the next state from $s$. We consider policies that only
access state $s$ through its feature vector $\Phi(s)$, meaning that
$\pi(s)$ is a mapping from $\mathbb{R}^d$ to the set of actions
$A(s)$. A \emph{trajectory} is a complete sequence of state/action
pairs from the starting state $b$ to an end state $e$. Trajectories
can be generated by repeatedly executing a policy $\pi$ in the search
space.  Without loss of generality, we assume the lengths of
trajectories are fixed and equal to $T$.  The expected loss of a
policy $J(\pi)$ is the expected loss of the end state of the
trajectory $e \sim \pi$, where $e \in E$ is an end state reached by
following the policy\footnote{Some imitation learning literature
  (e.g., \cite{ross11dagger,daume12coaching}) defines the loss of a
  policy as an accumulation of the costs of states and actions in the
  trajectory generated by the policy. For simplicity, we define the
  loss only based on the end state.  However, our theorems can be
  generalized.}.  Throughout, expectations are taken with respect to
draws of $(\bx,\by)$ from the training distribution, as well as any
internal randomness in the learning algorithm.

An optimal policy chooses the action leading to the minimal expected
loss at each state. For losses decomposable over the states in a
trajectory, generating an optimal policy is trivial given $\by^*$
(e.g., the sequence tagging example in \cite{daume09searn}). In
general, finding the optimal action at states not in the optimal
trajectory can be tricky (e.g.,
\cite{goldberg13oracles,goldberg14tabular}).%% We also define a
%% cost-to-go function $Q^\pi(s,a)$ which is the expected cost of taking
%% action $a$ in state $s$, and subsequently following $\pi$ to reach an
%% end state $e$.

\begin{figure}[t]
\centering
\eq{\includegraphics[width=0.75\linewidth]{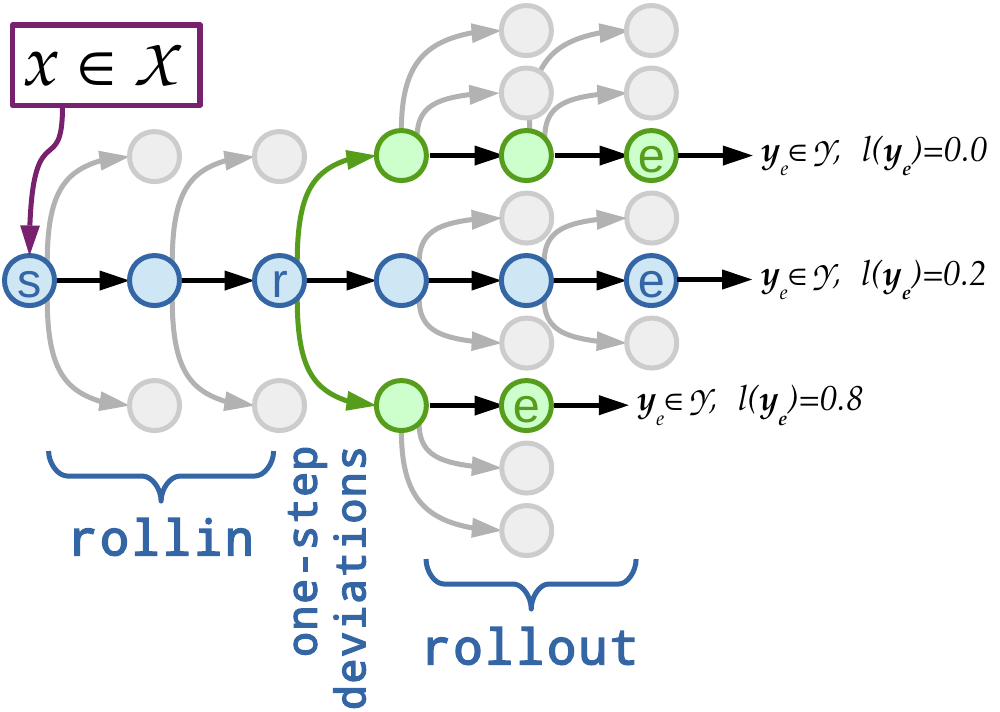}}
{\includegraphics[width=0.6\linewidth]{searchspace}}
\caption{An example search space. The exploration begins at the start
  state $s$ and chooses the middle among three actions by the
  \textbsf{roll-in} policy twice. Grey nodes are not explored. At state $r$ the learning algorithm
  considers the chosen action (middle) \emph{and} both one-step
  deviations from that action (top and bottom). Each of these
  deviations is completed using the \textbsf{roll-out} policy until an
  end state is reached, at which point the loss is collected. Here, we
  learn that deviating to the top action (instead of middle) at state
  $r$ decreases the loss by $0.2$.}
\label{fig:searchspace}
\end{figure}

\begin{algorithm}[t]
  \caption{Locally Optimal Learning to Search (\myalg)}
  \begin{algorithmic}[1]
    \REQUIRE {Dataset $\{\bx_i, \by_i\}_{i=1}^N$ drawn  from $\mathcal{D}
      $ and $\beta \geq 0$: a mixture parameter for roll-out.}
%% a search space $(S,b,E,A,l)$}, and  $\beta\geq 0$: a mixture
%%     parameter for roll-out.  \textbf{We don't define the tuple
%%       anymore, what to do with the search space here?}
    \STATE Initialize a policy $\pi_{0}$. 
	\FORALL{ $i \in \{1, 2, \ldots, N\} $ (loop over each instance)}
    \STATE Generate a reference policy $\piref$ based on $\by_i$. 
    %$\{s^0_0,s^0_1, s^0_2, \ldots s^0_T\}$ is the sequence of states that $\pi^0$ traverse through $\Omega$.
    \STATE Initialize $\trainset = \emptyset$. 
    \FORALL{$t \in \{0,1,2, \ldots, T-1\}$ }
    \STATE Roll-in by executing $\piin=\pil{i}$ for $t$ rounds and
    reach $s_t$.
    \FORALL{$a\in A(s_t)$}
    \STATE Let $\piout\!=\!\piref$ with probability $\beta$, otherwise $\pil{i}$.
    \STATE Evaluate cost $c_{i,t}(a)$ by rolling-out with $\piout$ for $T-t-1$ steps.
    \ENDFOR
    \STATE Generate a feature vector $\Phi(\bx_i, s_t)$.    
    \STATE Set $\trainset = \trainset \cup \{ \langle c_{i,t},\Phi(\bx_i,s_t) \rangle\}$. 
    \ENDFOR
    \STATE $ \pil{i+1} \leftarrow \text{Train}(\pil{i},\trainset)$ 
    (Update).  \label{line:update}
    \ENDFOR
    \STATE Return the average policy across $\pil{0},\pil{1}, \ldots \pil{N}$.
    %\STATE $\bw \leftarrow \b0, \AL \leftarrow \b0, \aset =\emptyset$, and an  inital cache $\Omega=\emptyset$
%\WHILE{stopping conditions are not satisfied}
%\STATE Let $\ilp{q}$ be thYi8RRYU84e ILP of Eq. 
%\eqref{eq:loss-augmented-inference} with $\bw, \bx_i, \by_i$. 
%\STATE $\bby \leftarrow$ InferenceSolver($\ilp{q}, \Omega$) \label{line:inf}
%\STATE Add $(\bby,\ilp{q})$ to $\Omega$.  
%\IF{$\alpha_{i,\bby} \notin \aset$ and Eq.  \eqref{eq:updateactive} holds}
%\STATE  $\aset \leftarrow \aset \cup \{\alpha_{i,\bby}\}$.
%  \ENDIF
%\ENDFOR
%\RETURN $\bw$
  \end{algorithmic}
  \label{alg:learning}
\end{algorithm}

%cost-sensitive learner

Finally, like most other \lts algorithms, \myalg assumes access to a
cost-sensitive classification algorithm. A cost-sensitive classifier
predicts a label $\hat{y}$ given an example $\bx$, and receives a
loss $\bc_{\bx}(\hat{y})$, where $\bc_{\bx}$ is a vector containing the cost for
each possible label. In order to perform online updates, we
assume access to a no-regret online cost-sensitive learner, which we
formally define below. 

\begin{definition}
  Given a hypothesis class $\H~:~\mathcal{X}\rightarrow [K]$, the
  regret of an online cost-sensitive classification algorithm which
  produces hypotheses $h_1, \ldots, h_M$ on cost-sensitive example
  sequence $\{(\bx_1, \bc_1), \ldots, (\bx_M, \bc_M)\}$ is
  \begin{equation}
    \regret^{\CS}_M = \sum_{m=1}^M \bc_m(h_m(\bx_m)) - \min_{h \in
      \H} \sum_{m=1}^M \bc_m(h(\bx_m)).
    \label{eqn:cs-regret}
  \end{equation}
  An algorithm is no-regret if $\regret^{\CS}_M = o(M)$.
  \label{defn:no-regret}
\end{definition}

Such no-regret guarantees can be obtained, for instance, by applying
the SECOC technique~\cite{Langford2005} on top of any importance
weighted binary classification algorithm that operates in an online
fashion, examples being the perceptron algorithm or online ridge
regression.

% training process
\myalg (see Algorithm \ref{alg:learning}) learns a policy $\pil{} \in
\Pi$ to approximately minimize $J(\pi)$,\footnote{ We can
  parameterize the policy $\pil{}$ using a weight vector $\bw \in
  \mathbb{R}^d$ such that a cost-sensitive classifier can be used to
  choose an action based on the features at each state. We do not
  consider using different weight vectors at different states.}
assuming access to a reference policy $\piref$ (which may or may not
be optimal). The algorithm proceeds in an online fashion generating a
sequence of learned policies $\pil{0}, \pil{1}, \pil{2}, \ldots$.  At
round $i$, a structured sample $(\bx_i, \by_i)$ is observed, and the
configuration of a search space is generated along with the reference
policy $\piref$. Based on $(\bx_i, \by_i)$, \myalg constructs $T$
cost-sensitive multiclass examples using a roll-in policy $\piin$ and
a roll-out policy $\piout$.  The roll-in policy is used to generate an
initial trajectory and the roll-out policy is used to derive the
expected loss.  More specifically, for each decision point $t \in
[0,T)$, \myalg executes $\piin$ for $t$ rounds reaching a state $s_t
  \sim \piin$.  Then, a cost-sensitive multiclass example is generated
  using the features $\Phi(s_t)$. Classes in the multiclass example
  correspond to available actions in state $s_t$.  The cost $c(a)$
  assigned to action $a$ is the difference in loss between taking
  action $a$ and the best action.
\begin{equation}
	c(a) = \ell(e(a)) - \min_{a'} \ell(e(a')),
\end{equation}
where $e(a)$ is the end state reached with rollout by $\piout$ after
taking action $a$ in state $s_t$. \myalg collects the $T$ examples
from the different roll-out points and feeds the set of examples
$\trainset$ into an online cost-sensitive multiclass learner, thereby
updating the learned policy from $\pil{i}$ to $\pil{i+1}$.  By
default, we use the learned policy $\pil{i}$ for roll-in and
a mixture policy for roll-out. For each roll-out, the
mixture policy either executes $\piref$ to an end-state with
probability $\beta$ or $\pil{i}$ with probability $1-\beta$.  \myalg
converts into a batch algorithm with a standard online-to-batch
conversion where the final model $\bar{\pi}$ is generated by
averaging $\pil{i}$ across all rounds (i.e., picking one of $\pil{1},
\ldots \pil{N}$ uniformly at random).

\section{Theoretical Analysis}
\label{sec:analysis}

In this section, we analyze \myalg and answer the questions raised in
Section \ref{sec:intro}. Throughout this section we use $\pibar$ to
denote the average policy obtained by first choosing $n \in [1,N]$ uniformly at random and then acting according to $\pi_n$.% $\pibar = (\pi_{1} + \ldots + \pil{N})/N$.
We
begin with discussing the choices of roll-in and roll-out policies.
%As we mentioned in Section \ref{sec:search}, roll-in policy is used to 
%	explore the search space and the roll-out policy is used to measure the structured loss.
Table \ref{tab:rollinrollout} summarizes the results of using
different strategies for roll-in and roll-out.

\subsection{The \emph{Bad} Choices}
An obvious \emph{bad} choice is roll-in and roll-out with the
learned policy, because the learner is blind to the
reference policy.  It reduces the structured learning problem to a
reinforcement learning problem, which is much harder. To build
intuition, we show two other \emph{bad} cases.

\begin{table}
\begin{center}
\bgroup
\def\arraystretch{1.5}
\begin{tabular}{|l|c|c|c|}
\hline
{\cellcolor{Gray}}roll-out $\rightarrow$ & {\cellcolor{Gray}} & {\cellcolor{Gray}} & {\cellcolor{Gray}} \\
{\cellcolor{yellow!10}}$\downarrow$ roll-in &
\multirow{-2}{1.8cm}{\centering\bf Reference}\cellcolor{Gray} &
\multirow{-2}{1.4cm}{\centering\bf Mixture}\cellcolor{Gray} &
\multirow{-2}{1.4cm}{\centering\bf Learned}\cellcolor{Gray} \\  
\hline
\cellcolor{yellow!10}\bf Reference & \multicolumn{3}{c|}{Inconsistent\cellcolor{red!10}} 
                                   \\ \hline
%\cellcolor{yellow!10}\bf Mixture     & Not locally \cellcolor{red!10}
%                                   & \cellcolor{green!10}
%                                   & ?? \cellcolor{Gray}
%                                   \\ \cline{1-1}\cline{4-4}
\cellcolor{yellow!10}\bf Learned   & Not locally opt. \cellcolor{red!10}
%                                   & \multirow{-2}{*}{Good} \cellcolor{green!10}
                                   & {Good} \cellcolor{green!10}
									& RL \cellcolor{red!10}
                                   \\
\hline
\end{tabular}
\egroup
\end{center}
\caption{Effect of different roll-in and roll-out policies. The
  strategies marked with ``\emph{Inconsistent}'' might generate a
  learned policy with a large structured regret, and the strategies
  marked with ``\emph{Not locally opt.}'' could be much worse than its
  one step deviation. The strategy marked with ``\emph{RL}'' reduces the 
  structure learning problem to a reinforcement learning problem, which is much 
  harder.  
  The strategy marked with ``\emph{Good}'' is
  favored.  }
\label{tab:rollinrollout}
\end{table}

\tikzstyle{level 1}=[level distance=0.6cm, sibling distance=1.2cm]
\tikzstyle{level 2}=[level distance=0.6cm, sibling distance=0.6cm]
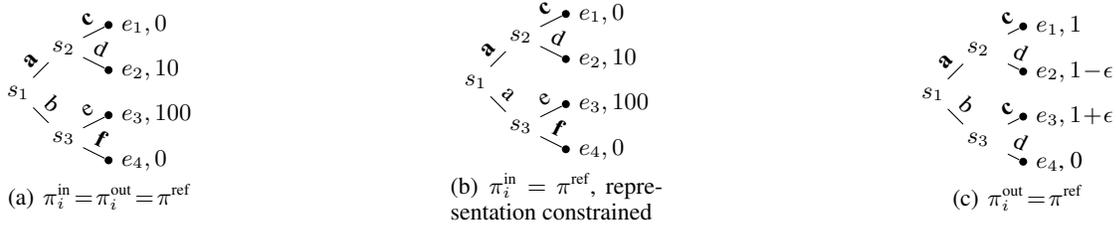
\begin{figure}[t!]
{\small
\centering
  %\begin{subfigure}{}
 \eq{ \begin{tabular}{@{\hspace{-10pt}}c@{\quad}c@{ }c@{}}}
 { \begin{tabular}{@{\hspace{25pt}}ccc}}
  \begin{minipage}{0.33\linewidth}
    \subfigure[$\piin\!=\!\piout\!=\!\piref$]{
      \begin{tikzpicture}[grow=right, sloped]
	\node[]{$s_1$}
	child {
	  node{$s_3$}
	  child {
	    node[end, label=right:
	      {$e_4, 0$}] {}
	    edge from parent
	    node[above] {{\bf f}}
	  }
	  child {
	    node[end, label=right:
	      {$e_3, 100$}] {}
	    edge from parent
	    node[above] {e}
	  }
	  edge from parent 
	  node[above] {b}
	}
	child {
	  node{$s_2$}
	  child {
	    node[end, label=right:
	      {$e_2, 10$}] {}
	    edge from parent
	    node[above] {d}
	  }
	  child {
	    node[end, label=right:
	      {$e_1, 0$}] {}
	    edge from parent
	    node[above] {{\bf c}}
	  }
	  edge from parent         
	  node[above] {{\bf a}}
	};
      \end{tikzpicture}
      \label{fig:inref}
    }
  \end{minipage}
  %\caption{Roll in= Ref}
  %\end{subfigure}
  &
  \begin{minipage}{0.33\linewidth}
    \subfigure[\small $\piin\!=\!\piref$, representation constrained]{
      \begin{tikzpicture}[grow=right, sloped]
	\node[]{$s_1$}
	child {
	  node{$s_3$}
	  child {
	    node[end, label=right:
	      {$e_4, 0$}] {}
	    edge from parent
	    node[above] {{\bf f}}
	  }
	  child {
	    node[end, label=right:
	      {$e_3, 100$}] {}
	    edge from parent
	    node[above] {e}
	  }
	  edge from parent 
	  node[above] {a}
	}
	child {
	  node{$s_2$}
	  child {
	    node[end, label=right:
	      {$e_2, 10$}] {}
	    edge from parent
	    node[above] {d}
	  }
	  child {
	    node[end, label=right:
	      {$e_1, 0$}] {}
	    edge from parent
	    node[above] {{\bf c}}
	  }
	  edge from parent         
	  node[above] {{\bf a}}
	};
      \end{tikzpicture}
      \label{fig:inrefrepresentation}
    }
  \end{minipage}
  &
  %\begin{subfigure}{}
  \begin{minipage}{0.33\linewidth}
    \subfigure[$\piout\!=\!\piref$]{
      \begin{tikzpicture}[grow=right, sloped]
	\node[]{$s_1$}
	child {
	  node[bag]{$s_3$}
	  child {
	    node[end, label=right:
	      {$e_4$, 0}] {}
	    edge from parent
	    node[above] {d}
	  }
	  child {
	    node[end, label=right:
	      {$e_3, 1\!+\!\epsilon$}] {}
	    edge from parent
	    node[above] {{\bf c}}
	  }
	  edge from parent 
	  node[above] {b}
	}
	child {
	  node[bag]{$s_2$}
	  child {
	    node[end, label=right:
	      {$e_2, 1\!-\!\epsilon$}] {}
	    edge from parent
	    node[above] {d}
	  }
	  child {
	    node[end, label=right:
	      {$e_1, 1$}] {}
	    edge from parent
	    node[above] {{\bf c}}
	  }
	  edge from parent         
	  node[above] {{\bf a}}
	};
      \end{tikzpicture}
      \label{fig:outref}
      %\caption{Roll out= Ref}
      %\end{subfigure}
    }
  \end{minipage}
  \end{tabular}

  \caption{Counterexamples of $\piin=\piref$ and $\piout=\piref$.  All
    three examples have 7 states. The loss of each end state is
    specified in the figure.  A policy chooses actions to traverse
    through the search space until it reaches an end state. Legal
    policies are bit-vectors, so that a policy with a weight on $a$
    goes up in $s_1$ of Figure~\ref{fig:inref} while a weight on $b$
    sends it down. Since features uniquely identify actions of the
    policy in this case, we just mark the edges with corresponding
    features for simplicity. The reference policy is
    bold-faced.  In Figure~\ref{fig:inrefrepresentation}, the features
    are the same on either branch from $s_1$, so that the learned
    policy can do no better than pick randomly between the two. In
    Figure \ref{fig:outref}, states $s_2$ and $s_3$ share the same
    feature set (i.e., $\Phi(s_2) = \Phi(s_3)$).  Therefore, a policy
    chooses the same set of actions at states $s_2$ and $s_3$.  Please
    see text for details.  }
  \label{fig:counterexample}
  }
\end{figure}

%\includegraphics{example_3}

%\subsection{Rollin and Rollout Strategies}

{\bf Roll-in with $\piref$ is \emph{bad}.}  Roll-in with a reference
policy causes the state distribution to be unrealistically good.  As a
result, the learned policy never learns to correct for
previous mistakes, performing poorly when testing.  A related
discussion can be found at Theorem 2.1 in \cite{ross10reductions}. We
show a theorem below.
\begin{theorem}
  For $\piin = \piref$, there is a distribution $D$ over $(\bx, \by)$
  such that the induced cost-sensitive regret $\regret^{\CS}_M = o(M)$
  but $J(\pibar) - J(\piref) = \Omega(1)$.
\end{theorem}
\begin{proof}
  We demonstrate examples where the claim is true.

  We start with the case where $\piout = \piin = \piref$. In this
  case, suppose we have one structured example, whose search space is
  defined as in Figure \ref{fig:inref}. From state $s_1$, there are
  two possible actions: $a$ and $b$ (we will use actions and features
  interchangeably since features uniquely identify actions here); the
  (optimal) reference policy takes action $a$.  From state $s_2$,
  there are again two actions ($c$ and $d$); the reference takes
  $c$. Finally, even though the reference policy would never visit
  $s_3$, from that state it chooses action $f$.  When rolling in with
  $\piref$, the cost-sensitive examples are generated only at state
  $s_1$ (if we take a one-step deviation on $s_1$) and $s_2$ but
  \emph{never} at $s_3$ (since that would require a two deviations,
  one at $s_1$ and one at $s_3$). As a result, we can never learn how
  to make predictions at state $s_3$.  Furthermore, under a rollout
  with $\piref$, both actions from state $s_1$ lead to a loss of zero.
  The learner can therefore learn to take action $c$ at state $s_2$
  and $b$ at state $s_1$, and achieve \emph{zero} cost-sensitive
  regret, thereby ``thinking'' it is doing a good job.  Unfortunately,
  when this policy is actually run, it performs as badly as possible
  (by taking action $e$ half the time in $s_3$), which results in the
  large structured regret.

  Next we consider the case where $\piout$ is either the learned
  policy or a mixture with $\piref$. When applied to the example in
  Figure~\ref{fig:inrefrepresentation}, our feature representation is
  not expressive enough to differentiate between the two actions at
  state $s_1$, so the learned policy can do no better than pick
  randomly between the top and bottom branches from this state. The
  algorithm either rolls in with $\piref$ on $s_1$ and generates a
  cost-sensitive example at $s_2$, or generates a cost-sensitive
  example on $s_1$ and then completes a roll out with
  $\piout$. Crucially, the algorithm still never generates a
  cost-sensitive example at the state $s_3$ (since it would have
  already taken a one-step deviation to reach $s_3$ and is constrained
  to do a roll out from $s_3$). As a result, if the learned policy
  were to choose the action $e$ in $s_3$, it leads to a zero
  cost-sensitive regret but large structured regret.\end{proof}

  Despite these negative results, rolling in with the learned policy
  is robust to both the above failure modes. In
  Figure~\ref{fig:inref}, if the learned policy picks action $b$ in
  state $s_1$, then we can roll in to the state $s_3$, then generate a
  cost-sensitive example and learn that $f$ is a better action than
  $e$. Similarly, we also observe a cost-sensitive example in $s_3$ in
  the example of Figure~\ref{fig:inrefrepresentation}, which clearly
  demonstrates the benefits of rolling in with the learned policy as
  opposed to $\piref$.

{\bf Roll-out with $\piref$ is \emph{bad} if $\piref$ is not
  optimal. } When the reference policy is not optimal \emph{or} the
reference policy is not in the hypothesis class, roll-out with
$\piref$ can make the learner blind to compounding errors.  The
following theorem holds. We state this in terms of ``local
optimality'': a policy is locally optimal if changing any \emph{one}
decision it makes never improves its performance.
\begin{theorem}
  For $\piout = \piref$, there is a distribution $D$ over $(\bx, \by)$
  such that the induced cost-sensitive regret $\regret^{\CS}_M = o(M)$
  but $\pibar$ has arbitrarily large structured regret to one-step deviations.
\end{theorem}
\begin{proof}
  Suppose we have only one structured example, whose search space is
  defined as in Figure \ref{fig:outref} and the reference policy
  chooses $a$ or $c$ depending on the node. If we roll-out with
  $\piref$, we observe expected losses 1 and $1+\epsilon$ for actions
  $a$ and $b$ at state $s_1$, respectively.  Therefore, the policy
  with zero cost-sensitive classification regret chooses actions $a$
  and $d$ depending on the node.  However, a one step deviation
  ($a\rightarrow b$) does radically better and can be learned by
  instead rolling out with a mixture policy.
\end{proof}

The above theorems show the bad cases and motivate a good $\lts$
algorithm which generates a learned policy that competes with the
reference policy and deviations from the learned policy.  In the
following section, we show that Algorithm \ref{alg:learning} is such
an algorithm.

\subsection{Regret Guarantees}
\label{sec:regret}
Let $Q^\pi(s_{t}, a)$ represent the expected loss of executing action
$a$ at state $s_t$ and then executing policy $\pi$ until reaching an
end state. $T$ is the number of decisions required before reaching an
end state.  For notational simplicity, we use $Q^\pi(s_t,\pi')$ as a
shorthand for $Q^\pi(s_t,\pi'(s_t))$, where $\pi'(s_t)$ is the action
that $\pi'$ takes at state $s_t$. Finally, we use $\state{t}{\pi}$ to
denote the distribution over states at time $t$ when acting according
to the policy $\pi$.
%$\Qref(s_t,\pi')$ represents the expected loss of rolling out with
%$\piref$. 
The expected loss of a policy is:
\begin{equation}
	\label{eq:j}
J(\pi) = \E_{s\sim\state{t}{\pi}}\left[Q^{\pi}(s,\pi)\right],
\end{equation}
for any $t \in [0,T]$. In words, this is the expected cost of rolling
in with $\pi$ up to some time $t$, taking $\pi$'s action at time $t$
and then completing the roll out with $\pi$. 

Our main regret guarantee for Algorithm~\ref{alg:learning} shows that
\myalg minimizes a combination of regret to the reference
policy $\piref$ and regret its own one-step deviations. %% Let $\ell^*$ denote
%% the minimum attainable loss for the online cost-sensitive learner,
%% that is
%% %
%% \begin{equation*}
%%   \ell^* = \min_{\pi \in
%%     \Pi} \frac{1}{NT} \sum_{i=1}^N \sum_{t=1}^T \E_{s
%%     \sim \state{t}{\pil{i}}} [Q^{\piout}(s, \pi)],
%% \end{equation*}
%% %
%% where $\piout = \beta \piref + (1 - \beta) \pil{i}$ is the policy with
%% which we roll out at iteration $i$. 
In order to concisely present the result, we present an additional
definition which captures the regret of our approach:
%% \begin{small}
%% \begin{align}
%%   \nonumber \epsref &= \frac{1}{NT} \sum_{i=1}^N \sum_{t=1}^T \E_{s
%%     \sim \state{t}{\pil{i}}} [Q^{\piout}(s, \pil{i})] - \ell^*, ~ \mbox{and}\\
%%   \nonumber \epsclass &= \ell^* - \frac{1}{NT} \sum_{i,t}  \E_{s \sim \state{t}{\pil{i}}} 
%%       \left[\beta \min_a \Qref(s, a) \right.\\&\qquad\qquad\left.+
%%         (1-\beta)\min_a Q^{\pil{i}}(s,a)\right].  
%%   \label{eqn:epsilons}
%% \end{align}
%% \end{small}
\eq{
  \begin{small}
\begin{align}
  \label{eqn:epsilons}
  \nonumber \epsfull &= \frac{1}{NT} \sum_{i=1}^N \sum_{t=1}^T \E_{s
    \sim \state{t}{\pil{i}}} \left[Q^{\piout}(s, \pil{i}) -
  \left(\beta \min_a \Qref(s, a) \right. \right.\\&\qquad\qquad\left. \left.+
  (1-\beta)\min_a Q^{\pil{i}}(s,a) \right)\right],
\end{align}
\end{small}
}{
\begin{equation}
  \label{eqn:epsilons}
  \epsfull = \frac{1}{NT} \sum_{i=1}^N \sum_{t=1}^T \E_{s
    \sim \state{t}{\pil{i}}} \left[Q^{\piout}(s, \pil{i}) -
  \left(\beta \min_a \Qref(s, a) +
  (1-\beta)\min_a Q^{\pil{i}}(s,a) \right)\right],
\end{equation}
}
where $\piout = \beta \piref + (1-\beta)\pil{i}$ is the mixture policy
used to roll-out in Algorithm~\ref{alg:learning}. With these
definitions in place, we can now state our main result for
Algorithm~\ref{alg:learning}.

\begin{theorem}
  \label{thm:regret}
  Let $\epsfull$ be as defined in Equation~\ref{eqn:epsilons}.
  The averaged policy $\pibar$ generated by running $N$ steps of
  Algorithm~\ref{alg:learning} with a mixing parameter $\beta$
  satisfies
%  \eq{
  \begin{small}
    \begin{align*}
    \beta (J(\pibar) - J(\piref)) + (1 - \beta) & \sum_{t=1}^T\big(
    J(\pibar) - \min_{\pi \in \Pi} \E_{s \sim \state{t}{\pibar}}
    [Q^{\pibar}(s,\pi)]\big) \\ &\leq T\epsfull.
    \end{align*}
  \end{small}
%   }{
%     \begin{equation}
%    \beta (J(\pibar) - J(\piref)) + (1 - \beta)  \sum_{t=1}^T\big(
%     J(\pibar) - \min_{\pi \in \Pi} \E_{s \sim \state{t}{\pibar}}
%     [Q^{\pibar}(s,\pi)]\big) \leq T\epsfull.
%     \end{equation}
%  }
\end{theorem}

%% The difference $J(\bar{\pi}) - J(\piref)$ can be written as the sum of $T$-one step deviations and hence the two terms on the left hand side are on the same scale. 
It might appear that the LHS of the theorem combines one term which is
constant to another scaling with $T$. We point the reader to
Lemma~\ref{lemma:regret-onestep} in the appendix to see why the terms
are comparable in magnitude.  Note that the theorem does not assume
anything about the quality of the reference policy, and it might be
arbitrarily suboptimal. Assuming that Algorithm~\ref{alg:learning}
uses a no-regret cost-sensitive classification algorithm (recall
Definition~\ref{defn:no-regret}), the first term in the definition of
$\epsfull$ converges to
\begin{equation*}
  \ell^* = \min_{\pi \in \Pi} \frac{1}{NT} \sum_{i=1}^N \sum_{t=1}^T
  \E_{s \sim \state{t}{\pil{i}}} [Q^{\piout}(s, \pi)].
\end{equation*}
This observation is formalized in the next corollary.

\begin{corollary}
Suppose we use a no-regret cost-sensitive classifier in
Algorithm~\ref{alg:learning}. As $N \rightarrow \infty$, $\epsfull
\rightarrow \epsclass$, where
\eq{
\begin{align}
  \nonumber \epsclass = \ell^* - &\frac{1}{NT} \sum_{i,t} \E_{s \sim
    \state{t}{\pil{i}}} \left[\beta \min_a \Qref(s, a)
    \right.\\&\qquad\qquad\left.+ (1-\beta)\min_a
    Q^{\pil{i}}(s,a)\right]. \nonumber
%  \label{eqn:epsclass}
\end{align}
}{
\begin{equation*}
  \nonumber \epsclass = \ell^* - \frac{1}{NT} \sum_{i,t} \E_{s \sim
    \state{t}{\pil{i}}} \left[\beta \min_a \Qref(s, a)
    + (1-\beta)\min_a
    Q^{\pil{i}}(s,a)\right]. 
%  \label{eqn:epsclass}
\end{equation*}
}
\label{cor:noregret}
\end{corollary}
When we have $\beta = 1$, so that \myalg becomes almost identical to
\aggrevate~\cite{ross14aggrevate}, $\epsclass$ arises solely due to
the policy class $\Pi$ being restricted. For other values of $\beta
\in (0,1)$, the asymptotic gap does not always vanish even if the
policy class is unrestricted, since $\ell^*$ amounts to obtaining
$\min_a Q^{\piout}(s, a)$ in each state. This corresponds to taking a
minimum of an average rather than the average of the corresponding
minimum values. 

In order to avoid this asymptotic gap, it seems desirable to have
regrets to reference policy and one-step deviations controlled
individually, which is equivalent to having the guarantee of
Theorem~\ref{thm:regret} for all values of $\beta$ in $[0,1]$ rather
than a specific one. As we show in the next section, guaranteeing
a regret bound to one-step deviations when the reference policy is
arbitrarily bad is rather tricky and can take an exponentially long
time. Understanding structures where this can be done more tractably
is an important question for future research. Nevertheless, the result
of Theorem~\ref{thm:regret} has interesting consequences in several
settings, some of which we discuss next.

\begin{enumerate}[itemsep=0.1em,topsep=0em,parsep=0.1em]
\item The second term on the left in the theorem is always
  non-negative by definition, so the conclusion of
  Theorem~\ref{thm:regret} is at least as powerful as existing regret
  guarantee to reference policy when $\beta = 1$.  Since the previous
  works in this area~\cite{daume09searn,ross11dagger,ross14aggrevate}
  have only studied regret guarantees to the reference policy, the
  quantity we're studying is strictly more difficult. 

\item The asymptotic regret incurred by using a mixture policy for
  roll-out might be larger than that using the reference policy alone,
  when the reference policy is near-optimal.  How the combination of
  these factors manifests in practice is empirically evaluated in
  Section~\ref{sec:exp}.

\item When the reference policy is optimal, the first term is
  non-negative. Consequently, the theorem demonstrates that our
  algorithm competes with one-step deviations in this case. This is
  true irrespective of whether $\piref$ is in the policy class $\Pi$
  or not.

  %% While this
  %% is clearly a consequence of having small regret to $\piref$ when
  %% $\piref$ is contained in the policy class, small regret to $\piref$
  %% does not directly imply small regret to one-step deviations when
  %% $\piref$ is outside the policy class. However,
  %% Theorem~\ref{thm:regret} yields a non-trivial local optimality
  %% guarantee even in this case.

\item When the reference policy is very suboptimal, then the first
  term can be negative. In this case, the regret to one-step
  deviations can be large despite the guarantee of
  Theorem~\ref{thm:regret}, since the first negative term allows the
  second term to be large while the sum stays bounded. However, when
  the first term is significantly negative, then the learned policy
  has already improved upon the reference policy substantially! This
  ability to improve upon a poor reference policy by using a mixture
  policy for rolling out is an important distinction for
  Algorithm~\ref{alg:learning} compared with previous approaches.%% ,
  %% since rolling out with the reference policy only effective precludes
  %% the possibility of improving beyond a poor reference policy.
\end{enumerate}

Overall, Theorem~\ref{thm:regret} shows that the learned policy is
either competitive with the reference policy \emph{and} nearly locally
optimal, or improves substantially upon the reference policy.

\subsection{Hardness of local optimality}

In this section we demonstrate that the process of reaching a local
optimum (under one-step deviations) can be exponentially slow when the
initial starting policy is arbitrary. This reflects the hardness of
learning to search problems when equipped with a poor reference
policy, even if local rather than global optimality is considered a
yardstick. We establish this lower bound for a class of algorithms
substantially more powerful than \myalg.  We start by defining a
search space and a policy class. Our search space consists of
trajectories of length $T$, with $2$ actions available at each step of
the trajectory. We use $0$ and $1$ to index the two actions. We
consider policies whose only feature in a state is the depth of the
state in the trajectory, meaning that the action taken by any policy
$\pi$ in a state $s_t$ depends only on $t$. Consequently, each policy
can be indexed by a bit string of length $T$. For instance, the policy
$0100\ldots0$ executes action $0$ in the first step of any trajectory,
action $1$ in the second step and $0$ at all other levels. It is
easily seen that two policies are one-step deviations of each other if
the corresponding bit strings have a Hamming distance of 1.

To establish a lower bound, consider the following powerful
algorithmic pattern. Given a current policy $\pi$, the algorithm
examines the cost $J(\pi')$ for all the one-step deviations $\pi'$ of
$\pi$. It then chooses the policy with the smallest cost as its new
learned policy. Note that access to the actual costs $J(\pi)$ makes
this algorithm more powerful than existing \lts algorithms, which can
only estimate costs of policies through rollouts on individual
examples. Suppose this algorithm starts from an initial policy
$\pil{0}$.  How long does it take for the algorithm to reach a policy
$\pil{i}$ which is locally optimal compared with all its one-step
deviations?  We next present a lower bound for algorithms of this style.

\begin{theorem}
  Consider any algorithm which updates policies only by moving from
  the current policy to a one-step deviation. Then there is a search
  space, a policy class and a cost function where the any such
  algorithm must make $\Omega(2^T)$ updates before reaching a locally
  optimal policy. Specifically, the lower bound also applies to
  Algorithm~\ref{alg:learning}. 
  \label{thm:snakes}
\end{theorem}

The result shows that competing with the seemingly reasonable
benchmark of one-step deviations may be very challenging from an
algorithmic perspective, at least without assumptions on the search
space, policy class, loss function, or starting policy.  For instance,
the construction used to prove Theorem~\ref{thm:snakes} does not apply
to Hamming loss.

\section{Structured Contextual Bandit}
\label{sec:structbandit}
We now show that a variant of \myalg can be run in a ``structured
contextual bandit'' setting, where only the loss of a single
structured label can be observed.  As mentioned, this setting has
applications to webpage layout, personalized search, and several other
domains.

At each round, the learner is given an input example $\bx$, makes a
prediction $\hat{\by}$ and suffers structured loss $\ell(\by^*,
\hat{\by})$.  We assume that the structured losses lie in the interval
$[0,1]$, that the search space has depth $T$ and that there are at
most $K$ actions available at each state.  As before, the algorithm
has access to a policy class $\Pi$, and also to a reference policy
$\piref$.  It is important to emphasize that the reference policy does
not have access to the true label, and the goal is improving on the
reference policy.

\begin{algorithm}[t]
  \caption{Structured Contextual Bandit Learning}
  \begin{algorithmic}[1]
    \REQUIRE {Examples $\{\bx_i\}_{i=1}^N$, reference policy $\piref$,
      exploration probability $\epsilon$ and  mixture parameter
      $\beta\geq 0$.}    
    \STATE Initialize a policy $\pi_{0}$, and set $\explset =
    \emptyset$.  
    \FORALL{ $i=1,2,\ldots,N$ (loop over each instance)}
    \STATE Obtain the example $\bx_i$, set $\explore = 1$ with
    probability $\epsilon$, set $\numexpl{i} = |\explset|$. 
    \IF{$\explore$}
    \STATE Pick random time $t \in \{0,1, \ldots, T-1\}$.
    \STATE Roll-in by executing $\piin=\pil{\numexpl{i}}$ for $t$
    rounds and reach $s_t$.
    \STATE Pick random action $a_t \in A(s_t)$; let $K = |A(s_t)|$.
    \STATE Let $\piout = \piref$ with probability $\beta$, otherwise
    $\pil{\numexpl{i}}$. 
    \STATE Roll-out with $\piout$ for $T-t-1$ steps to evaluate 
    \begin{equation*}
      \hat{c}(a) = K \ell(e(a_t)) \mathbf{1}[a = a_t].
%    \label{eqn:costhat}
    \end{equation*}
    \STATE Generate a feature vector $\Phi(\bx_i, s_t)$.    
    \STATE $ \pil{\numexpl{i}+1} \leftarrow \text{Train}(\pil{\numexpl{i}},\hat{c},
    \Phi(\bx_i, s_t))$.  
    \label{line:cbupdate}
    \STATE Augment $\explset = \explset \cup \{\pil{\numexpl{i}+1}\}$
    \ELSE
    \STATE Follow the trajectory of a policy $\pi$ drawn randomly from
    $\explset$ to an end state $e$, predict the corresponding
    structured output $\by_{ie}$.
    \ENDIF
    \ENDFOR
  \end{algorithmic}
  \label{alg:structcb}
\end{algorithm}

Our approach is based on the $\epsilon$-greedy algorithm which is a
common strategy in partial feedback problems. Upon receiving an
example $\bx_i$, the algorithm randomly chooses whether to
\emph{explore or exploit} on this example. With probability
$1-\epsilon$, the algorithm chooses to exploit and follows the
recommendation of the current learned policy. With the remaining
probability, the algorithm performs a randomized variant of the \myalg
update.  A detailed description is given in Algorithm~\ref{alg:structcb}.

We assess the algorithm's performance via a measure of regret, where
the comparator is a mixture of the reference policy and the best
one-step deviation.  Let $\pibar_i$ be the averaged policy based on
all policies in $\explset$ at round $i$. $\by_{ie}$ is the predicted
label in either step 9 or step 14 of Algorithm~\ref{alg:structcb}. The
average regret is defined as: \eq{
\begin{align*}
\textrm{Regret} &= \frac{1}{N}\sum_{i=1}^N \Big(\E[\ell(\by^*_i, \by_{ie})] -
\beta \E[\ell(\by^*_i, \by_{i e_{\textrm{ref}}})] \\ 
& - (1-\beta) \sum_{t=1}^T \min_{\pi \in \Pi} \E_{s \sim
  d_{\bar{\pi}_i}^t}[Q^{\bar{\pi}_i}(s,\pi)]\Big)
\end{align*}
}{
\begin{equation*}
\textrm{Regret} = \frac{1}{N}\sum_{i=1}^N \Big(\E[\ell(\by^*_i, \by_{ie})] -
\beta \E[\ell(\by^*_i, \by_{i e_{\textrm{ref}}})]  - (1-\beta) \sum_{t=1}^T \min_{\pi \in \Pi} \E_{s \sim
  d_{\bar{\pi}_i}^t}[Q^{\bar{\pi}_i}(s,\pi)]\Big)
\end{equation*}
}
Recalling our earlier definition of
$\epsfull[i]$~\eqref{eqn:epsilons}, we bound on the regret of
Algorithm~\ref{alg:structcb} with a proof in the appendix.
\begin{theorem}
\label{thm:regretscb}
  Algorithm~\ref{alg:structcb} with parameter $\epsilon$ satisfies:
  \begin{align*}
    \textrm{Regret} \le \epsilon + \frac{1}{N} \sum_{i=1}^N
    \epsfull[\numexpl{i}],
  \end{align*}
  \label{thm:structcb}
\end{theorem}

With a no-regret learning algorithm, we expect
\begin{equation}
  \epsfull[i] \leq \epsclass + cK\sqrt{\frac{\log|\Pi|}{i}},
  \label{eqn:noregretcb}
\end{equation}
where $|\Pi|$ is the cardinality of the policy class.  This
leads to the following corollary with a proof in the appendix.

\begin{corollary}
  \label{corscb}
  In the setup of Theorem~\ref{thm:structcb}, suppose further that the
  underlying no-regret learner satisfies~\eqref{eqn:noregretcb}. Then
  with probability at least $1 - 2/(N^5K^2T^2\log(N|\Pi|))^3$, 
  \begin{align*}
  \textrm{Regret} = O\left((KT)^{2/3} \sqrt[3]{\frac{\log (N|\Pi|)}{N}} +
  T\epsclass\right). 
\end{align*}
\end{corollary}

\section{Experiments}
\label{sec:exp}

This section shows that \myalg is able to improve upon a suboptimal reference 
policy and provides empirical evidence to support the analysis in Section \ref{sec:analysis}.
We conducted experiments on the following three applications.

\noindent {\bf Cost-Sensitive Multiclass classification.} 
For each cost-sensitive multiclass sample, each choice of label has an 
associated cost. The search space for this task is a binary search tree.
The root of the tree corresponds to the whole set of labels. We recursively split the set of labels in half, until 
each subset contains only one label. A trajectory through the search space is a path from root-to-leaf in this tree.
The loss of the end state is defined by the cost. 
%If the goal node corresponds to the gold label, the loss is $0$; otherwise, it is $1$.
An optimal reference policy can lead the agent to the end state with the 
minimal cost.  
We also show results of using a bad reference policy which arbitrarily chooses an action at each state. 
The experiments are conducted on {\sf KDDCup 99} dataset\footnote{\url{
http://kdd.ics.uci.edu/databases/kddcup99/kddcup99.html}} generated 
from a computer network intrusion detection task.  The dataset contains $5$ classes, $4,898,431$ training and $311,029$ test instances.

\begin{table}[t]
\begin{center}
\bgroup
\def\arraystretch{1}
\begin{tabular}{|l|c|c|c|}
\hline
\cellcolor{Gray}roll-out $\rightarrow$ & \cellcolor{Gray} & \cellcolor{Gray} & \cellcolor{Gray} \\
\cellcolor{yellow!10}$\downarrow$ roll-in & \multirow{-2}{1.6cm}{\centering\bf Reference}\cellcolor{Gray} & \multirow{-2}{1.6cm}{\centering\bf Mixture}\cellcolor{Gray} & \multirow{-2}{1.6cm}{\centering\bf Learned}\cellcolor{Gray} \\
%& oracle & mixture & learn \\
\hline
\multicolumn{4}{|c|}{\bf Reference is optimal} \\
\hline
\cellcolor{yellow!10} {\bf Reference}& 0.282 &  0.282 &   0.279   \\
\cellcolor{yellow!12} {\bf Learned}& 0.267 & \cellcolor{green!10} \best{0.266} &  \best{0.266}   \\
\hline
\multicolumn{4}{|c|}{\bf Reference is bad} \\
\hline
\cellcolor{yellow!10} {\bf Reference}& 1.670 &  1.664 &  0.316   \\
\cellcolor{yellow!10} {\bf Learned} & \best{0.266} & \cellcolor{green!10} \best{0.266} &  \best{0.266}  \\
\hline
\end{tabular}
\egroup
\end{center}
\caption{ The average cost on cost-sensitive classification dataset;
  columns are roll-out and rows are roll-in.  The best result is
  bold. \textbf{\searn achieves 0.281 and 0.282 when the reference
    policy is optimal and bad, respectively.} \myalg is
  Learned/Mixture and highlighted in green.  }
\label{tab:mc}
\end{table}

\begin{table}[t]
\begin{center}
\bgroup
\def\arraystretch{1}
\begin{tabular}{|l|c|c|c|}
\hline
\cellcolor{Gray}roll-out $\rightarrow$ & \cellcolor{Gray} & \cellcolor{Gray} & \cellcolor{Gray} \\
\cellcolor{yellow!10}$\downarrow$ roll-in & \multirow{-2}{1.6cm}{\centering\bf Reference}\cellcolor{Gray} & \multirow{-2}{1.6cm}{\centering\bf Mixture}\cellcolor{Gray} & \multirow{-2}{1.6cm}{\centering\bf Learned}\cellcolor{Gray} \\
\hline
\multicolumn{4}{|c|}{\bf Reference is optimal} \\
\hline                     %    REF         MIX     LEARN
\cellcolor{yellow!10} {\bf Reference}&  95.58   &  94.12  & 94.10 \\
\cellcolor{yellow!10} {\bf Learned} &  \best{95.61}   &  \cellcolor{green!10} 94.13  & 94.10 \\
\hline
% \multicolumn{4}{|c|}{\bf Reference is suboptimal} \\
% \hline
% \cellcolor{yellow!10} {\bf Reference}&  95.32   &  94.04  & 94.04 \\
% \cellcolor{yellow!10} {\bf Learned} &  \best{95.61}   &  94.13  & 94.10 \\
% \hline
% \multicolumn{4}{|c|}{\bf Reference is bad} \\
% \hline
% \cellcolor{yellow!10} {\bf Reference}&  95.55   & 94.17  &  94.15 \\
% \cellcolor{yellow!10} {\bf Learned} &  \best{95.61}   & 94.13  &  94.10 \\
% \hline
\end{tabular}
\egroup
\end{center}
\caption{The accuracy on POS tagging;
columns are roll-out and rows are roll-in.  
The best result is bold. 
 \textbf{\searn achieves 94.88.}
\myalg is Learned/Mixture and highlighted in green.
}
\label{tab:pos}
\end{table}

\begin{table}[t]
\begin{center}
\bgroup
\def\arraystretch{1}
\begin{tabular}{|l|c|c|c|}
\hline
\cellcolor{Gray}roll-out $\rightarrow$ & \cellcolor{Gray} & \cellcolor{Gray} & \cellcolor{Gray} \\
\cellcolor{yellow!10}$\downarrow$ roll-in & \multirow{-2}{1.6cm}{\centering\bf Reference}\cellcolor{Gray} & \multirow{-2}{1.6cm}{\centering\bf Mixture}\cellcolor{Gray} & \multirow{-2}{1.6cm}{\centering\bf Learned}\cellcolor{Gray} \\
\hline
\multicolumn{4}{|c|}{\bf Reference is optimal} \\
\hline
\cellcolor{yellow!10} {\bf Reference}& 87.2 &  89.7 &  88.2   \\
\cellcolor{yellow!10} {\bf Learned}& \best{90.7}  &\cellcolor{green!10} 90.5  &  86.9  \\
%Searn &\multicolumn{3}{|c|}{84.0} \\
\hline
\multicolumn{4}{|c|}{\bf Reference is suboptimal} \\
\hline
\cellcolor{yellow!10} {\bf Reference}& 83.3 & 87.2 &  81.6   \\
\cellcolor{yellow!10} {\bf Learned}& 87.1 & \cellcolor{green!10}\best{90.2}   &  86.8  \\
%Searn &\multicolumn{3}{|c|}{81.1} \\
\hline
\multicolumn{4}{|c|}{\bf Reference is bad} \\
\hline
\cellcolor{yellow!10} {\bf Reference}&68.7&  65.4 &  66.7   \\
\cellcolor{yellow!10} {\bf Learned}& 75.8  & \cellcolor{green!10}\best{89.4}  & 87.5 \\
%Searn &\multicolumn{3}{|c|}{83.9} \\
\hline
\end{tabular}
\egroup
\end{center}
\caption{The UAS score on dependency parsing data set; columns are
  roll-out and rows are roll-in.  The best result is
  bold. \textbf{\searn achieves 84.0, 81.1, and 63.4 when the
    reference policy is optimal, suboptimal, and bad, respectively.}
  \myalg is Learned/Mixture and highlighted in green.  }
\label{tab:dep}
\end{table}

\noindent {\bf Part of speech tagging.}  The search space for POS
tagging is left-to-right prediction. Under Hamming loss the trivial
optimal reference policy simply chooses the correct part of speech for
each word. We train on $38k$ sentences and test on $11k$ from the Penn
Treebank~\cite{marcus93treebank}. One can construct suboptimal or even
bad reference policies, but under Hamming loss these are all
equivalent to the optimal policy because roll-outs by any fixed policy
will incur exactly the same loss and the learner can immediately learn
from one-step deviations.

\noindent {\bf Dependency parsing.} A dependency parser learns to
generate a tree structure describing the syntactic dependencies
between words in a
sentence~\cite{mcdonald05dependency,nivre03parsing}. We implemented a
hybrid transition system~\cite{kuhlmann2011dynamic} which parses a
sentence from left to right with three actions: \textsc{Shift},
\textsc{ReduceLeft} and \textsc{ReduceRight}. We used the
``non-deterministic oracle''~\cite{goldberg13oracles} as the optimal
reference policy, which leads the agent to the best end state
reachable from each state.  We also designed two suboptimal reference
policies. A bad reference policy chooses an arbitrary legal action at
each state.  A suboptimal policy applies a greedy selection and
chooses the action which leads to a good tree when it is obvious;
otherwise, it arbitrarily chooses a legal action. (This suboptimal
reference was the \emph{default} reference policy used prior to the
work on ``non-deterministic oracles.'')  We used data from the Penn
Treebank Wall Street Journal corpus: the standard data split for
training (sections 02-21) and test (section 23). The loss is evaluated
in UAS (unlabeled attachment score), which measures the fraction of
words that pick the correct parent.

For each task and each reference policy, we compare 6 different
combinations of roll-in (learned or reference) and roll-out (learned,
mixture or reference) strategies. We also include \searn in the
comparison, since it has notable differences from \myalg.  \searn
rolls in and out with a mixture where a different policy is drawn for
each state, while \myalg draws a policy once per example. \searn uses
a batch learner, while \myalg uses online.  The policy in \searn is a
mixture over the policies produced at each iteration.  For \myalg, it
suffices to keep just the most recent one.  It is an open research
question whether an analogous theoretical guarantee of Theorem
\ref{thm:regret} can be established for \searn.

Our implementation is based on Vowpal
Wabbit\footnote{\url{http://hunch.net/~vw/}}, a machine learning
system that supports online learning and \lts.  For \myalg's mixture
policy, we set $\beta=0.5$. We found that \myalg is not sensitive to $\beta$,
and setting $\beta$ to be 0.5 works well in practice.
 For \searn, we set the mixture parameter
to be $1-(1-\alpha)^t$, where $t$ is the number of rounds and
$\alpha=10^{-5}$. Unless stated otherwise all the learners take 5
passes over the data.

Tables \ref{tab:mc}, \ref{tab:pos} and \ref{tab:dep} show the results
on cost-sensitive multiclass classification, POS tagging and
dependency parsing, respectively.  The empirical results qualitatively
agree with the theory.  Rolling in with reference is always bad.  When
the reference policy is \textbf{optimal}, then doing roll-outs with
reference is a good idea.  However, when the reference policy is
\textbf{suboptimal} or \textbf{bad}, then rolling out with reference
is a bad idea, and mixture rollouts perform substantially
better. \myalg also significantly outperforms \searn on all tasks.

\section{Proofs of Main Results}

\begin{lemma}[Ross \& Bagnell Lemma 4.3] For any two policies, $\pi_{1},\pi_{2}$: 

  $J(\pi_{1})-J(\pi_{2}) =
  T\E_{t\sim U(1,T),s\sim\state{t}{\pi_{1}}}\left[Q^{\pi_{2}}(s,\pi_{1})-Q^{\pi_{2}}(s,\pi_{2})\right]
  =
  T\E_{t\sim U(1,T),s\sim\state{t}{\pi_{2}}}\left[Q^{\pi_{1}}(s,\pi_{1})-Q^{\pi_{1}}(s,\pi_{2})\right]$
  \label{lemma:regret-onestep}
\end{lemma}

\begin{proof}
Let $\pi^{t}$ be a policy that executes $\pi_{1}$ in the first $t$
steps and then executes $\pi_{2}$ from time steps $t+1$ to $T$. We
have $J(\pi_{1})=J(\pi^{T})$ and
$J(\pi_{2})=J(\pi^{0})$. Consequently, we can set up the telescoping
sum: 
\begin{align*}
J(\pi_{1})-J(\pi_{2}) &=
\sum_{t=1}^{T}\left[J(\pi^{t})-J(\pi^{t-1})\right] =
\sum_{t=1}^{T}\left[\E_{s\sim\state{t}{\pi_{1}}}
  \left[Q^{\pi_{2}}(s_{t,}\pi_{1})-Q^{\pi_{2}}(s_{t,}\pi_{2})\right]\right]\\
&=
T\E_{t\sim U(1,T), s\sim\pi_{1}}\left[Q^{\pi_{2}}(s,\pi_{1})-Q^{\pi_{2}}(s,\pi_{2})\right]
\end{align*}

The second equality in the lemma can be obtained by reversing the
roles of $\pi_1$ and $\pi_2$ above. 
\end{proof}

\subsection{Proof of Theorem~\ref{thm:regret}}

We start with an application of
Lemma~\ref{lemma:regret-onestep}. Using the lemma, we have:
\begin{align}
  \nonumber J(\pibar)-J(\piref) &=
  \frac{1}{N}\sum_{i} \left[ J(\pil{i}) - J(\piref)\right] \\ 
  &= \frac{1}{N}\sum_{i} \left[ T\E_{t\sim U(1,T), s\sim\pil{i}} \left[
      \Qref(s, \pil{i}) - \Qref(s, \piref)\right]\right]
  \label{eqn:refreg}
\end{align}
We also observe that
\begin{align}
  &\sum_{t=1}^T \left(J(\pibar) - \min_{\pi \in \Pi} \E_{s \sim
    \state{t}{\pibar}} [Q^{\pibar}(s,\pi)]\right) \nonumber\\  
  &= \frac{1}{N} \sum_{i=1}^N \sum_{t=1}^T \left[ J(\pil{i}) -
    \min_{\pi \in \Pi} \E_{s \sim \state{t}{\pil{i}}}
        [Q^{\pil{i}}(s,\pi)]\right] \nonumber\\
  &= \frac{1}{N} \sum_{i=1}^N \sum_{t=1}^T \left[ \E_{s \sim
      \state{t}{\pil{i}}} [Q^{\pil{i}}(s,\pil{i})] - \min_{\pi
      \in \Pi} \E_{s \sim \state{t}{\pil{i}}}
    [Q^{\pil{i}}(s,\pi)]\right] \nonumber\\
  &\leq \frac{1}{N} \sum_{i=1}^N \sum_{t=1}^T \left[ \E_{s \sim
      \state{t}{\pil{i}}} [Q^{\pil{i}}(s,\pil{i}) - \min_a
      Q^{\pil{i}}(s,a)]\right]. 
  \label{eqn:localreg}
\end{align}
Combining the above bounds from Equations~\ref{eqn:refreg}
and~\ref{eqn:localreg}, we see that
\begin{align*}
  & \beta \left(J(\pibar) - J(\piref)\right) + (1 - \beta) 
  \sum_{t=1}^T\left( J(\pibar) - \min_{\pi \in \Pi} \E_{s \sim
    \state{t}{\pibar}} [Q^{\pibar}(s,\pi)]\right)\\
  \leq & \frac{1}{N} \sum_{i=1}^N \sum_{t=1}^T \E_{s \sim
      \state{t}{\pil{i}}} \left[ \beta\left( \Qref(s, \pil{i}) -
    \Qref(s, \piref) \right) + (1 - \beta) \left(
    Q^{\pil{i}}(s,\pil{i}) - \min_a Q^{\pil{i}}(s,a) \right)
    \right]\\
  = & \frac{1}{N} \sum_{i=1}^N \sum_{t=1}^T \E_{s \sim
      \state{t}{\pil{i}}} \left[ Q^{\piout}(s, \pil{i}) -
    \beta \Qref(s, \piref) - (1 - \beta)\min_a Q^{\pil{i}}(s,a)
    \right]\\
  \leq & \frac{1}{N} \sum_{i=1}^N \sum_{t=1}^T \E_{s \sim
      \state{t}{\pil{i}}} \left[ Q^{\piout}(s, \pil{i}) -
    \beta \min_a \Qref(s, a) - (1 - \beta)\min_a Q^{\pil{i}}(s,a) 
    \right]
\end{align*}

\subsection{Proof of Corollary~\ref{cor:noregret}}

The proof is fairly straightforward from definitions. By definition of
no-regret, it is immediate that the gap
\begin{align}
  \sum_{i=1}^N \sum_{t=1}^T \E \left[c_{i,t}(\pil{i}(s_t)) -
    c_{i,t}(\pi(s_t))\right] = o(NT),
  \label{eqn:noregret}
\end{align}
for all policies $\pi \in \Pi$, where we recall that $c_{i,t}$ is the
cost-vector over the actions on round $i$ when we do roll-outs from
the $t_{th}$ decision point. Let $\E_i$ denote the conditional
expectation on round $i$, conditioned on the previous rounds in
Algorithm~\ref{alg:learning}. Then it is easily seen that
\begin{align*}
  \E_i[c_{i,t}(a)] &= \E_i\left[\ell(e_{i,t}(a)) - \min_{a^{'}}
    \ell(e_{i,t}(a^{'})) \right],
\end{align*}
with $e_{i,t}$ being the end-state reached on completing the roll-out
with the policy $\piout$ on round $i$, when action $a$ was taken on
the decision point $t$. Recalling that we rolled in following the
trajectory of $\piin$, this expectation further simplifies to 
\begin{align*}
  \E_i[c_{i,t}(a)] &= \E_{s \sim \state{t}{\pil{i}}}
  \left[Q^{\piout}(s, a)\right] - \E_i\left[\min_{a^{'}}
    \ell(e_{i,t}(a^{'})) \right]. 
\end{align*}
Now taking expectations in Equation~\ref{eqn:noregret} and combining
with the above observation, we obtain that for any policy $\pi \in
\Pi$, 
\begin{align*}
  &\sum_{i=1}^N \sum_{t=1}^T \E\left[c_{i,t}(\pil{i}(s_t)) -
    c_{i,t}(\pi(s_t))\right]\\ 
  &= \sum_{i=1}^N \sum_{t=1}^T \E_{s \sim
    \state{t}{\pil{i}}}\left[Q^{\piout}(s, \pil{i}(s)) -
    Q^{\piout}(s, \pi(s))\right] = o(NT).
\end{align*}
Taking the best policy $\pi \in \Pi$ and dividing through by $NT$
completes the proof.

\subsection{Proof sketch of Theorem~\ref{thm:regretscb}}
(Sketch only)
  We decompose the analysis over exploration and exploitation
  rounds. For the exploration rounds, we bound the regret by its
  maximum possible value of 1. To control the regret on the
  exploitation rounds, we focus on the updates performed during
  exploration. 

  The cost vector $\hat{c}(a)$ used at an exploration round $i$
  satisfies
  \begin{align*}
    \E_i[\hat{c}(a)] &= \E_i\left[K \ell(e(a_t)) \mathbf{1}[a = a_t]
      \right]\\
    &= \E_{t \sim U(0:T-1), s \sim \state{t}{\pil{\numexpl{i}}}}
    \left[ Q^{\piout}(s,a) \right],
  \end{align*}
  % using a similar argument as the proof of
  Corollary~\ref{cor:noregret}. Since the cost vector is identical in
  expectation as that used in Algorithm~\ref{alg:learning}, the proof
  of theorem~\ref{thm:regret}, which only depends on expectations, can
  be reused to prove a result similar to theorem~\ref{thm:regret} for
  the exploration rounds.  That is, letting $\pibar_i$ to be the
  averaged policy over all the policies in $\explset$ at exploration
  round $i$, we have the bound
    \begin{small}
    \begin{align*}
      \beta (J(\pibar_i) - J(\piref)) + (1 - \beta) & \sum_{t=1}^T\big(
      J(\pibar) - \min_{\pi \in \Pi} \E_{s \sim \state{t}{\pibar}}
      [Q^{\pibar}(s,\pi)]\big) \\ &\leq T\epsfull[i],
    \end{align*}
  \end{small}
where $\epsfull[i]$ is as defined in Equation~\ref{eqn:epsilons}.
  
  On the exploitation rounds, we can now invoke this
  guarantee. Recalling that we have $\numexpl{i}$ exploration rounds
  until round $i$, the expected regret at an exploitation round $i$ is
  at most $\delta_{\numexpl{i}}$. Thus the overall regret of the
  algorithm is at most 
  \begin{align*}
    \textrm{Regret} &\leq \epsilon + \frac{1}{N}\sum_{i=1}^N
    \epsfull[\numexpl{i}],
  \end{align*}
  which completes the proof.

\subsection{Proof of corollary~\ref{corscb}}
  We start by substituting Equation~\ref{eqn:noregretcb} in the regret
  bound of Theorem~\ref{thm:structcb}. This yields
  \begin{align*}
    \textrm{Regret} &\leq \epsilon + \frac{T}{N}\epsclass +
    \frac{cKT}{N}\sum_{i=1}^N \sqrt{\frac{\log|\Pi|}{\numexpl{i}}}.
  \end{align*}
  We would like to further replace $n_i$ with its expectation which is
  $\epsilon i$. However, this does not yield a valid upper bound
  directly. Instead, we apply a Chernoff bound to the quantity $n_i$,
  which is a sum of $i$ i.i.d. Bernoulli random variables with mean
  $\epsilon$. Consequently, we have 
  \begin{align*}
    \P\left(n_i \leq (1 - \gamma)\epsilon i \right) \leq
    \exp\left(-\frac{\gamma^2\epsilon i}{2} \right) \leq
    \exp(-\epsilon i/8),
  \end{align*}
  for $\gamma = 1/2$. Let $i_0 = 16\log N/\epsilon+1$. Then we can sum
  the failure probabilities above for all $i \geq i_0$ and obtain 
  \begin{align*}
    \sum_{i = i_0}^N \P\left(n_i \leq \epsilon i/2 \right) &\leq
    \sum_{i = i_0}^N \exp(-\epsilon i/8) \leq \sum_{i=i_0}^\infty
    \exp(-\epsilon i/8)\\
    &\leq \frac{\exp(-\epsilon i_0/8)}{1 - \exp(-\epsilon/8)}\\
    &= \frac{\exp(-2\log N)}{\exp(\epsilon/8) - 1} \leq
    \frac{8}{N^2\epsilon}, 
  \end{align*}
  where the last inequality uses $1 + x \leq \exp(x)$. Consequently,
  we can now allow a regret of 1 on the first $i_0$ rounds, and
  control the regret on the remaining rounds using $n_i \leq \epsilon
  i/2$. Doing so, we see that with probability at least $1 -
  2/(N^2\epsilon)$ 
  \begin{align*}
    \textrm{Regret} &\leq \epsilon + \frac{i_0}{N} +
    \frac{T}{N}\epsclass + \frac{cKT}{N}\sum_{i=1}^N
    \sqrt{\frac{2\log|\Pi|}{\epsilon i}}\\
    &\leq \epsilon + \frac{16 \log N + \epsilon}{N\epsilon} +
    \frac{T}{N}\epsclass + \frac{8cKT\log |\Pi|}{\epsilon N} 
   \end{align*}
  Choosing $\epsilon = (KT)^{2/3}(\log(N|\Pi|)/N)^{1/3}$ completes the
  proof. 

  \subsection{Proof of Theorem~\ref{thm:snakes}}

  \begin{figure*}[t]
    \centering
    \begin{tabular}{cccc}
      \begin{minipage}{0.2\textwidth}
        \scalebox{0.33}{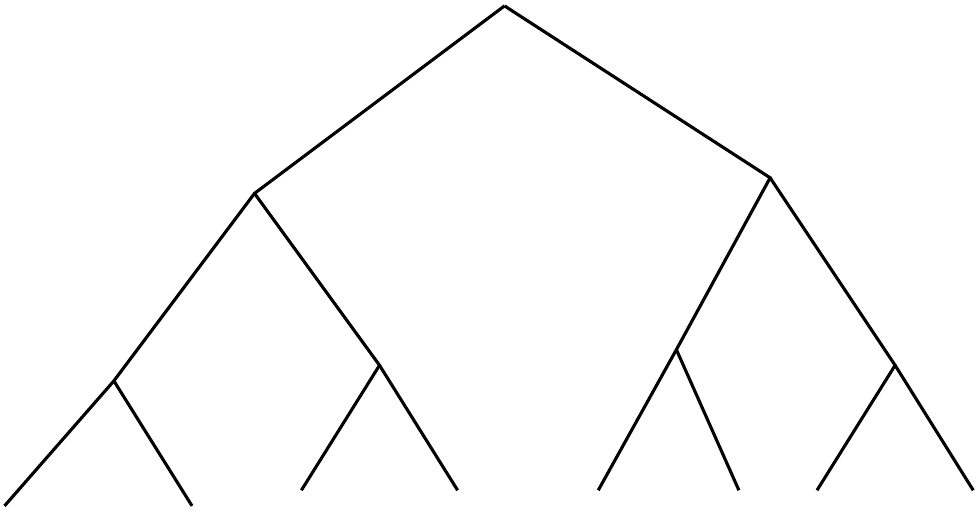}
      \end{minipage} &
      \begin{minipage}{0.2\textwidth}
        \scalebox{0.24}{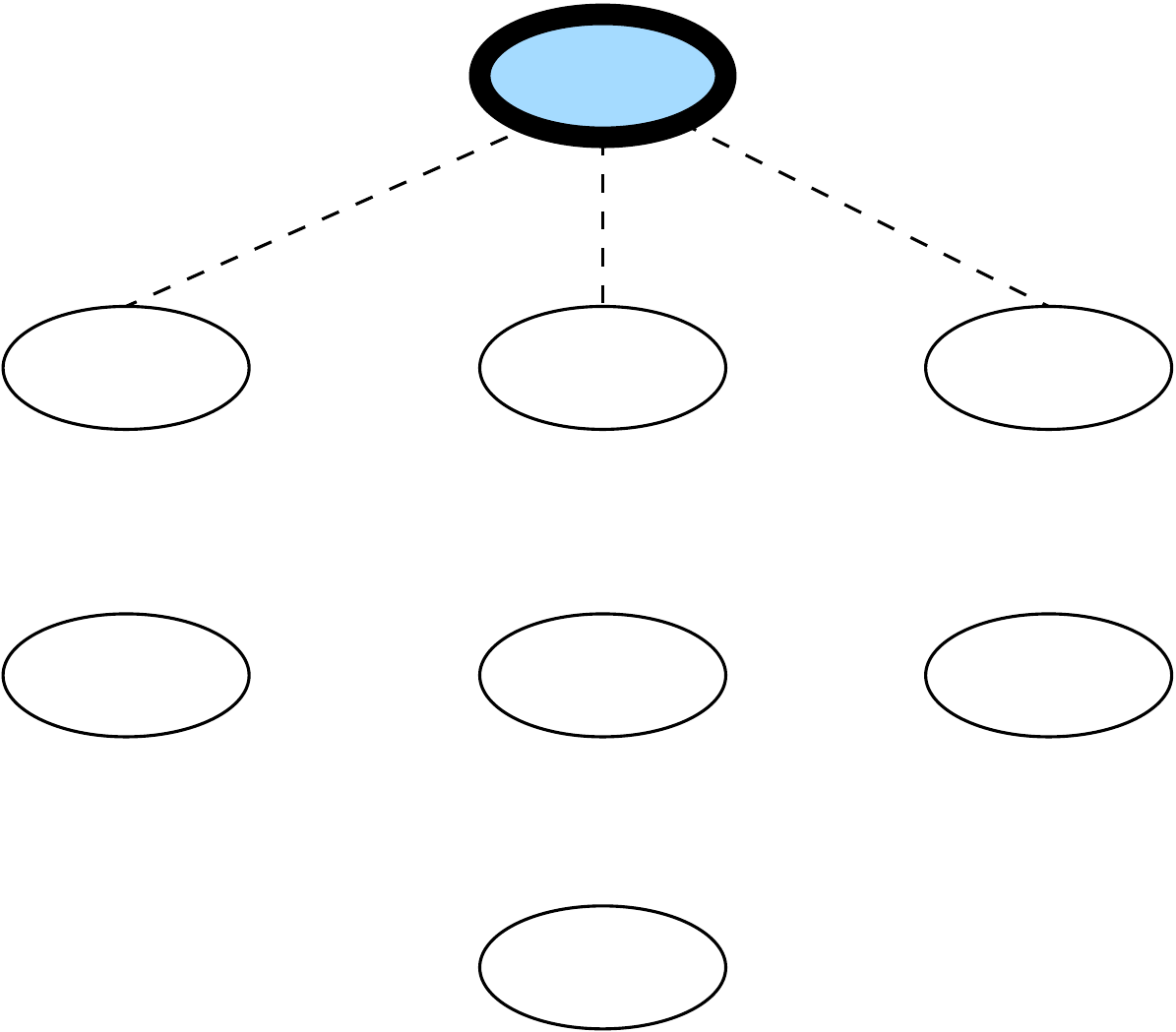}
      \end{minipage} &
      \begin{minipage}{0.2\textwidth}
        \scalebox{0.23}{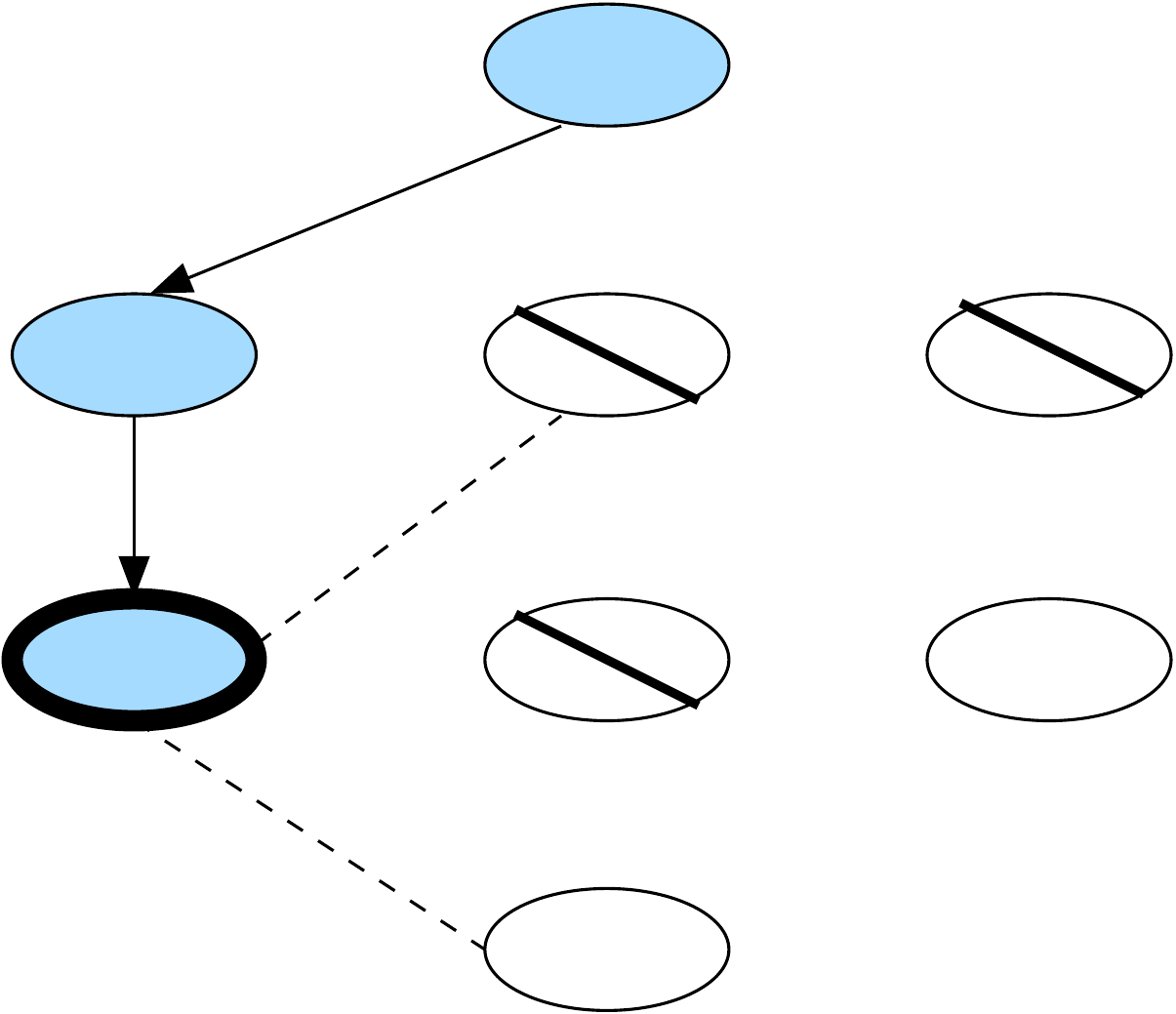}
      \end{minipage} &
      \begin{minipage}{0.2\textwidth}
        \scalebox{0.23}{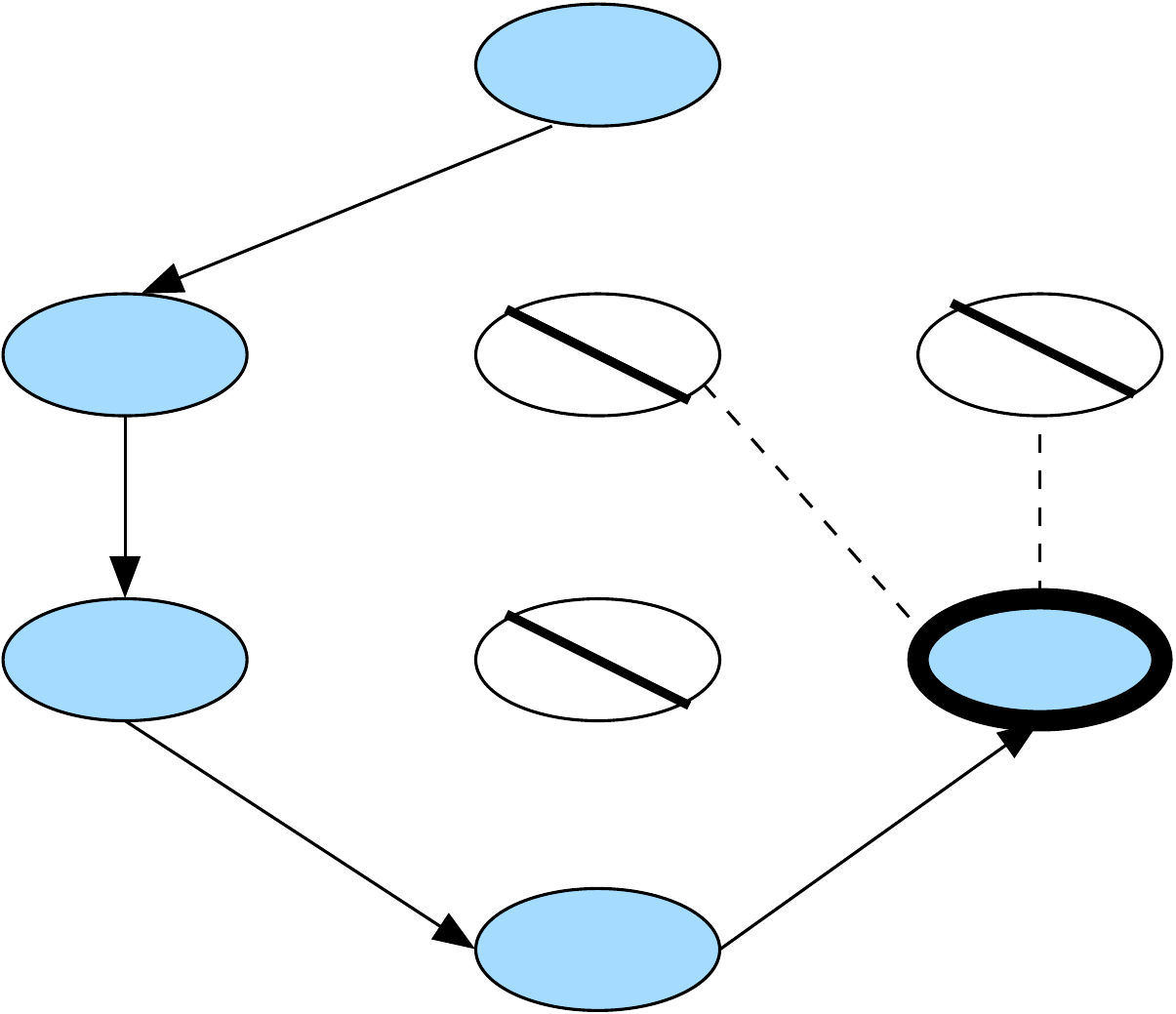}
      \end{minipage}\\
      (a) & (b) & (c) & (d) 
    \end{tabular}
    \caption{Pictorial illustration of the proof elements of
      Theorem~\ref{thm:snakes}. Panel (a) depicts the actions chosen by
      policy $000$. Selected action in each state is indicated in
      bold. Panels (b) through (d) depict various stages as the
      algorithm updates the policy to its one-step deviations, starting
      from the policy $000$. Each policy that the
      algorithm selects is depicted by a shaded circle, with the arrows
      marking the moves of the algorithm. Current policy is the shaded
      circle with a bold boundary. Dashed lines denote the potential
      one-step deviations that the algorithm can move to and crossed
      policies are those which have higher costs than the current policy
      (see text for details).}
    \label{fig:snakes}
  \end{figure*}
    
  The proof follows from results in combinatorics.  The dynamics of
  algorithms considered here can be thought of as a path through a
  graph where the vertices are the corners of the boolean hypercube in
  $T$ dimensions with two vertices at Hamming distance 1 sharing an
  edge.  We demonstrate that there is a cost function such that the
  algorithm is forced to traverse a long path before reaching a local
  optimum.  Without loss of generality, assume that the algorithm
  always moves to a one-step deviation with the lowest cost since
  otherwise longer paths exist.

  To gain some intuition, first consider $T = 3$ which is depicted in
  Figure~\ref{fig:snakes}. Suppose the algorithm starts from the
  policy $000$ then moves to the policy $001$. If the algorithm picks
  the best amongst the one-step deviations, we know that $J(001) \leq
  \min\{J(000), J(010), J(100)\}$, placing constraints on the costs of
  these policies which force the algorithm to not visit any of these
  policies later. Similarly, if the algorithm moves to the policy
  $011$ next, we obtain a further constraint $J(011) < \min\{J(101),
  J(001)\}$. It is easy to check that the only feasible move
  (corresponding to policies not crossed in
  Figure~\ref{fig:snakes}(c)) which decreases the cost under these
  constraints is to the policy $111$ and then $110$, at which point
  the algorithm attains local optimality since no more moves that
  decrease the cost are possible. In general, at any step $i$ of the
  path, the policy $\pil{i}$ is a one-step deviation of $\pil{i-1}$ and
  at least 2 or more steps away from $\pil{j}$ for $j <
  i-1$. The policy never moves to a neighbor of an ancestor (excluding
  the immediate parent) in the path.

  This property is the key element to understand more
  generally. Suppose we have a current path $\pil{1} \rightarrow
  \pil{2} \ldots \rightarrow \pil{i-1} \rightarrow \pil{i}$. Since we
  picked the best neighbor of $\pil{j}$ as $\pil{j+1}$, $\pil{i+1}$
  cannot be a neighbor of any $\pil{j}$ for $j < i$. Consequently, the
  maximum number of updates the algorithm must make is given by the
  length of the longest such path on a hypercube, where each vertex
  (other than start and end) neighbors exactly two other vertices on
  the path.  This is called the \emph{snake-in-the-box} problem in
  combinatorics, and arises in the study of error correcting codes. It
  is shown by~\citet{Abbott1988} that the length of longest such path
  is $\Theta(2^T)$.  With monotonically decreasing costs for policies
  in the path and maximal cost for all policies not in the path, the
  traversal time is $\Theta(2^T)$.

  Finally, it might appear that Algorithm~\ref{alg:learning} is
  capable of moving to policies which are not just one-step deviations
  of the currently learned policy, since it performs updates on
  ``mini-batches'' of $T$ cost-sensitive examples. However, on this
  lower bound instance, Algorithm~\ref{alg:learning} will be forced to
  follow one-step deviations only due to the structure of the cost
  function. For instance, from the policy $000$ when we assign maximal
  cost to policies $010$ and $100$ in our example, this corresponds to
  making the cost of taking action $1$ on first and second step very
  large in the induced cost-sensitive problem. Consequently, $001$ is
  the policy which minimizes the cost-sensitive loss even when all the
  $T$ roll-outs are accumulated, implying the algorithm is
  forced to traverse the same long path to local optimality.

\section*{Acknowledgements}
Part of this work was carried out while Kai-Wei, Akshay and Hal were
visiting Microsoft Research.

\bibliography{bibfile}
\bibliographystyle{icml2015}

\newpage

\appendix
\onecolumn
 
  \section{Details of cost-sensitive reduction}

  \begin{algorithm}[t]
  \caption{Cost-sensitive One Against All (CSOAA) Algorithm}
  \begin{algorithmic}[1]
    \REQUIRE Initial predictor $f_1(x)$
    \FORALL{$t = 1,2, \ldots T$}
    \STATE Observe $\{x_{t,i}\}_{i=1}^K$.
    \STATE Predict class $i_t = \arg\min_{i=1}^Kf_t(x_{t,i})$.
    \STATE Observe costs $\{c_{t,i}\}_{i=1}^K$.
    \STATE Update $f_t$ using online least-squares regression on data
    $\{x_{t,i}, c_{t,i}\}_{i=1}^K$. 
    \ENDFOR
  \end{algorithmic}
  \label{alg:csoaa}
  \end{algorithm}

  In this section we present the details of the reduction to
  cost-sensitive multiclass classification used in our experimental
  evaluation. The experiments used the Cost-Sensitive One Against All
  (CSOAA) classification technique, the pseudocode for which is
  presented in Algorithm~\ref{alg:csoaa}. In words, the algorithm
  takes as input a feature vector $x_{t,i}$ for class $i$ at round
  $t$. It then trains a regressor to predict the corresponding costs
  $c_{t,i}$ given the features. Given a fresh example, the predicted
  label is the one with the smallest predicted cost. This is a natural
  extension of the One Against All (OAA) approach for multiclass
  classification to cost-sensitive settings. Note that this also
  covers the alternative approach of having a common feature vector,
  $x_{t,i} \equiv z_t$ for all $i$ and instead training $K$ different
  cost predictors, one for each class. If $z_t \in \R^d$, one can
  simply create $x_{t,i} \in \R^{dK}$, with $x_{t,i} = z_t$ in the
  $i_{th}$ block and zero elsewhere. Learning a common predictor $f$
  on $x$ is now representationally equivalent to learning $K$ separate
  predictors, one for each class.

  There is one missing detail in the specification of
  Algorithm~\ref{alg:csoaa}, which is the update step. The specifics
  of this step depend on the form of the function $f(x)$ being
  used. For instance, if $f(x) = w^Tx$, then a simple update rule is
  to use online ridge regression (see e.g. Section 11.7
  in~\cite{CesaBianchi2006}). Online gradient
  descent~\cite{zinkevich03online} on the squared loss $\sum_{i=1}^K
  (f(x_{t,i}) - c_{t,i})^2$ is another simple alternative, which can
  be used more generally. The specific implementation in our
  experiments uses a more sophisticated variant of online gradient
  descent with linear functions.
  
  \section{Details of Experiments} 
  Our implementation is based on Vowpal Wabbit (VW) version 7.8 (\url{http://hunch.net/~vw/}). It is 
	available at \url{https://github.com/KaiWeiChang/vowpal_wabbit/tree/icmlexp}. 
	For \myalg, we use flags ``--search\_rollin'', ``--search\_rollout'', 	
	``--search\_beta'' to set the rollin policy, the rollout policy, and $\beta$, 
	respectively. We use ``--search\_interpolation policy 
	--search\_passes\_per\_policy --passes 5'' to enable \searn.  
	The details settings of various VW flags for the three experiments are shown below:
	\begin{itemize}
		\item POS tagging: we use ``--search\_task sequence --search 45 
			--holdout\_off --affix -2w,+2w --search\_neighbor\_features -1:w,1:w -b 28''
		\item Dependency parsing: we use `` --search\_task dep\_parser --search 
			12  --holdout\_off --search\_history\_length 3 
			--search\_no\_caching -b 24 --root\_label 8 --num\_label 12''
		\item Cost-sensitive multiclass: we use ``--search\_task multiclasstask 
			--search 5  --holdout\_off --mc\_cost''
	\end{itemize}
The data sets used in the experiments are available upon request.  

\begin{verbatim}

\end{verbatim}

\end{document}